\documentclass{article}
\usepackage[margin=1.3in]{geometry}

\usepackage{amsmath, amsthm, amssymb, graphicx, url}
\allowdisplaybreaks

\usepackage{times}
\usepackage{sectsty}
\sectionfont{\fontsize{12}{15}\selectfont}
\subsectionfont{\fontsize{10}{13}\selectfont}

\newtheorem{theorem}{Theorem}

\newtheorem{proposition}[theorem]{Proposition}
\newtheorem{lemma}[theorem]{Lemma}

\newtheorem{definition}[theorem]{Definition}

\usepackage{algorithm,algorithmic}
\usepackage[labelfont=it,margin=.5cm,]{caption}

\def\X{{\mathcal X}}
\def\H{{\mathcal H}}
\def\Y{{\mathcal Y}}
\def\A{{\mathcal A}}

\def\X{{\mathcal X}}
\def\H{{\mathcal H}}

\def\Y{{\mathcal Y}}
\def\W{{\mathcal W}}
\def\R{{\mathbb R}}

\newcommand{\cor}{\text{ } \mathrm{cor}}

\newcommand{\sign}{\text{ } \mathrm{sign}}

\newcommand{\y}{\ensuremath{\mathbf y}}
\newcommand{\p}{\ensuremath{\mathbf p}}

\newcommand{\K}{\ensuremath{\mathcal K}}

\def\x{\mathbf{x}}
\def\y{\mathbf{y}}

\newcommand{\ignore}[1]{}

\let\Pr\relax
\DeclareMathOperator{\Pr}{\mathbb{P}}

\newcommand{\E}{\mathbb{E}}

\newcommand{\eps}{\varepsilon}

\renewcommand{\leq}{~\le~}
\renewcommand{\geq}{~\ge~}

\let\oldtfrac\tfrac
\renewcommand{\tfrac}[2]{\smash{\oldtfrac{#1}{#2}}}

\let\nablaold\nabla
\renewcommand{\nabla}{\nablaold\mkern-2.5mu}

\title{\textbf{Online Agnostic Boosting via Regret Minimization}}
\author{
   Nataly Brukhim$^{1\,2}$ \and Xinyi Chen$^{1\,2}$  \and Elad Hazan$^{1\,2}$ \and Shay Moran$^{1}$\\
  \\
  $^1$ Google AI Princeton \\
  $^2$ Department of Computer Science, Princeton University \\
  \texttt{\{nbrukhim,xinyic,ehazan\}@princeton.edu}, \texttt{shaymoran@google.com}\\
}
\date{}

\begin{document}

\maketitle

\begin{abstract}%
  
Boosting is a widely used machine learning approach based on the idea of aggregating weak learning rules.
While in statistical learning numerous boosting methods exist both in the realizable and agnostic settings,
in online learning they exist only in the realizable case.
In this work we provide the first agnostic online boosting algorithm;
that is, given a weak learner with only marginally-better-than-trivial regret guarantees, 
our algorithm boosts it to a strong learner with sublinear regret.
  
  
  Our algorithm is based on an abstract (and simple) reduction to online convex optimization,
    which efficiently converts an arbitrary online convex optimizer to an online booster.
    Moreover, this reduction extends to the statistical as well as the online realizable settings,
    thus unifying the 4 cases of statistical/online and agnostic/realizable boosting.
  

\end{abstract}

\section{Introduction}

Boosting is a fundamental methodology in machine learning 
    which allows us to automatically convert (``boost'') a number of weak learning rules into a strong one.
    Boosting was first studied in the context of (realizable) PAC learning in a line of seminal works
    which include the celebrated Adaboost algorithm 
    as well an many other algorithms with various applications (see e.g.~\cite{Kearns88unpublished,Schapire90boosting,Freund90majority,Freund97decision}).
    It was later adapted to the agnostic PAC setting and was extensively studied in this context as well~\cite{bendavid2001agnostic,MansourM02,gavinsky2003optimally,KalaiS05,LongS08,kalai2008agnostic,kanade2009potential,feldman2009distribution, chen2016communication,freund1996game}.
    More recently, \cite{chenonline} and \cite{beygelzimer2015optimal} studied boosting in the context of online prediction
    and derived boosting algorithms in the realizable setting (a.k.a.\ mistake-bound model).

\vspace{2mm}

In this work we study agnostic boosting in the online setting:
    let $\H$ be a class of experts and assume we have an oracle access to a weak online learner for $\H$ with 
    a non-trivial (yet far from desired) regret guarantee.
    The goal is to use it to obtain a strong online learner for $\H$, i.e.\ which exhibits a vanishing regret.
    
\paragraph{Why Online Agnostic Boosting?}    
    The setting of realizable boosting poses a restriction on the possible input sequences: there must be an expert that attains near-zero mistake-bound on the input sequence. This is a non-standard assumption in online learning. 
    In contrast, in the (agnostic) setting we consider, there is {\it no restriction on the input sequence and it can be chosen  adversarially}. 
    
    
\paragraph{Applications of Online Agnostic Boosting.}
Apart from being a fundamental question in any machine learning setting,
    let us mention a couple of more concrete incentives 
    to study online agnostic boosting:
\begin{itemize}
    \item {\bf Differential Privacy and Online Learning:} A recent line of work revealed deep connections
    between online learning and differentially private learning \cite{Agarwal17dponline,Abernathy17onlilnedp,AlonLMM19,bousquet2019passing,Neel2018HowTU,Joseph2019TheRO,privacy,Bun20Privateonline}.
    In fact, these two notions are equivalent 
    in the sense that a class $\H$ can be PAC learned by a differentially private algorithm
    if and only if it can be learned in the online setting with vanishing regret~\cite{AlonLMM19,Bun20Privateonline}.
    However, the above equivalence is only known to hold from an information theoretic perspective, 
    and deriving \underline{efficient} reductions between online and private learning is an open problem~\cite{Neel2018HowTU}.
    The only case where an efficient reduction is known to exist
    is in converting a {\it pure private learner} to an online learner in the realizable setting~\cite{privacy}.
    This reduction heavily relies on the realizable-case online boosting algorithm by~\cite{beygelzimer2015optimal}.
    Moreover, the derivation of an agnostic online boosting algorithm is posed by \cite{privacy}
    as an open problem towards extending their reduction to the agnostic setting.
    
    \item {\bf Time Series Prediction and Online Control:} Recent machine learning literature considered the problem of controlling a dynamical system from the lens of online learning and regret minimization, see e.g.\ \cite{agarwal2019online,agarwal2019logarithmic, Hazan2019TheNC} and referenced work therein. The online learning approach also gave rise to the first boosting methods in this context \cite{agarwal2019boosting},
    and demonstrates the potential impact of boosting in the online setting. 
    Thus, the current work aims at continuing the development of the boosting methodology in online machine learning, 
    starting from the basic setting of expert advice.

\end{itemize}

\subsection{Main Results}

\paragraph{The Weak Learning Assumption.}
    In this paper we use the same formulation as \cite{kanade2009potential} used in the statistical setting. Towards this end, it is convenient to measure the performance of online learners using {\it gain} rather than loss:
    let $(x_1,y_1)\ldots (x_T,y_T)\in \X\times\{\pm 1\}$ be an (adversarial and adaptive) input sequence of examples presented to an online learning algorithm $\mathcal{A}$; that is, in each iteration $t=1\ldots T$,
    the adversary picks an example $(x_t,y_t)$, then
    the learner $\mathcal{A}$ first gets to observe $x_t$, 
    and predicts (possibly in a randomized fashion) $\hat y_t\in \{\pm 1\}$,
    and lastly it observes $y_t$ and gains a reward of $y_t\cdot \hat y_t$.
    The goal of the learner is to maximize the total gain (or correlation), 
    given by $\sum_t y_t\cdot \hat y_t$.
    Note that this is equivalent to the often used notion of {\it loss} 
    where in each iteration the learner suffers a loss of $1[y_t \neq \hat y_t]$
    and its goal is to minimize the accumulated loss $\sum_t 1[y_t\neq \hat y_t]$.
    \footnote{Indeed, $y_t \hat y_t =  1 - 2\cdot1[y_t \neq \hat y_t]$ since $y_t,\hat y_t\in\{\pm 1\}$.
    Therefore, the accumulated loss and correlation are affinely related by  
    $\sum y_t\cdot \hat y_t =  T - 2\cdot \sum_t 1[y_t\neq \hat y_t]$. }
    
\begin{definition}[Agnostic Weak Online Learning] \label{online_agnostic_wl}
Let $\H\subseteq \{\pm 1\}^{\X}$ be a class of experts, let $T$ denote the horizon length, and let $\gamma>0$ denote the advantage.
    An online learning algorithm $\W$ is a $(\gamma, T)$-\textbf{agnostic weak online learner (AWOL)} for~$\H$ 
    if for any sequence $(x_1, {y}_1), ...,(x_T, {y}_T)\in \X \times \{\pm 1\}$, at every iteration $t \in [T]$,
    the algorithm outputs $\W(x_t) \in \{\pm 1\}$ such that, 
\[
\E\Bigg[ \sum_{t=1}^T \W(x_t){y}_t\Bigg] \ge \gamma \text{ } \underset{h \in \H}{\max} \text{ } \E \Bigg[\sum_{t=1}^T h(x_t) {y}_t \Bigg]- R_{\W}(T),
\]
where the expectation is taken w.r.t the randomness of the weak learner $\W$ and that of the possibly adaptive adversary, $R_{\W}: \mathbb{N} \rightarrow \R_+$ 
is the additive regret: a non-decreasing, sub-linear function of~$T$. 
\end{definition}
Note the slight abuse of notation in the last definition: an online learner $\W$ is not an ``$\X\to\{\pm 1\}$'' function; rather it is an algorithm with an internal state that is updated as it is fed training examples. Thus, the prediction $\W(x_t)$ depends on the internal state of $\W$, and for notational convenience we avoid reference to the internal state.

\vspace{2mm}

Our agnostic online boosting algorithm has an oracle access to $N$ weak learners 
    and predicts each task by combining their predictions.
    The number of weak learners $N$ is a meta-parameter which can be tuned by the user
    according to the following trade-off: on the one hand, the regret bound improves as $N$ increases,
    and on the other hand, a larger number of weak learners is more costly in terms of computational resources.

\begin{theorem}[Agnostic Online Boosting]\label{thm:main}
Let $\H$ be a class of experts, let $T\in\mathbb{N}$ denote the horizon length,
    and let $\W_1,\ldots, \W_N$ be $(\gamma, T)$-\textbf{AWOL} for $\H$ 
    with advantage $\gamma$ and regret $R_{\W}(T)=o(T)$ (see Definition~\ref{online_agnostic_wl}).
    Then, there exists an online learning algorithm, which has oracle access to each $\W_i$, and has expected regret of at most
\[
\frac{R_{\W}(T)}{\gamma} + O\Bigl(\frac{T}{\gamma \sqrt{N}}\Bigr).
\]
\end{theorem}
To exemplify the interplay between $R_{\W}(\cdot)$ and $N$, 
    imagine a scenario where $R_{\W}(T) \approx \sqrt{T}$ (as is often the case for regret bounds).
    Then, setting the number of weak learners to be $N \approx T/\gamma^2$
    gives that the overall regret remains $\approx \sqrt{T}$.

\paragraph{An Abstract Framework for Boosting.}
Boosting and Regret Minimization algorithms are intimately related. 
    This tight connection is exhibited both in statistical boosting (see~\cite{freund1999adaptive,Freund97decision,Schapire2012})
    as well as in the online boosting (\cite{chenonline}).
    Our algorithm is inspired by this fruitful connection and utilizes it:
    in particular, Theorem \ref{thm:main} is an instantiation of a more abstract meta-algorithm 
    which takes an arbitrary {\it online convex optimizer} and uses it in a \underline{black-box} manner 
    to obtain an agnostic online boosting algorithm.
    Thus, in fact we obtain a family of boosting algorithms;
    one for each choice of an online convex optimizer.
    Specifically, Theorem \ref{thm:main} follows by picking {\it Online Gradient Decent} for the meta-algorithm.
    We present this in detail in Section \ref{sec:online}.

The same type of reasoning carries to realizable online boosting,
    and even to statistical boosting (both realizable and agnostic setting).
    In Section \ref{sec:general} we demonstrate a general reduction from each of these boosting settings
    to online convex optimization.

\subsection{Related Work}

As discussed above, \cite{chenonline} and \cite{beygelzimer2015optimal} studied online boosting in the realizable (mistake-bound) setting,
    while this work focuses on the agnostic (regret-bound) setting. 

\cite{gradientboosting} studies online boosting under real-valued loss functions.
    The main difference from our work is in the weak learning assumption:
    \cite{gradientboosting} consider weak learners that are in fact strong online learners 
    for a base class of regression functions. The boosting process
    produces an online learner for a bigger class which consists of the linear span of the base class.
    This is different from the setting considered here where the class is fixed, 
    but the regret bound is being boosted.

A main motivation in this work is the connection between boosting and regret minimization.
    This builds on and inspired by previous works that demonstrated this fruitful relationship.
    We refer the reader to the book by \cite{Schapire90boosting} (Chapter 6) 
    for an excellent presentation of this relationship in the context of Adaboost.

\subsection{Organization}
The main result of our agnostic online boosting algorithm, and the proof of Theorem \ref{thm:oab}, are given in
 Section \ref{sec:online}. In Section \ref{sec:general}, we first give a game-theoretic perspective of our method when applied to the statistical setting (Subsection \ref{zero_sum_games}). We then demonstrate a general reduction, in the statistical setting, from both the agnostic (Subsection \ref{subsec:ab}), and realizable (Subsection \ref{subsec:rb}) boosting settings, to online convex optimization.
Lastly, we give a similar result for the online realizable boosting setting in Section \ref{sec:online_realizable}.
    
\section{Agnostic Online Boosting}\label{sec:online}
In this section we prove Theorem~\ref{thm:main}, which establishes an efficient online agnostic boosting algorithm. 
    We begin in Subsection~\ref{sec:oab} with formally presenting our framework which enables converting an online convex optimizer to an online booster. 
    Then, in Subsection~\ref{sec:proofmain} we show how Theorem~\ref{thm:main} follows directly 
    by picking the online convex optimizer to be {Online Gradient Decent}.

\subsection{Online Agnostic Boosting with OCO}\label{sec:oab}
We begin with describing our boosting algorithm (see Algorithm~\ref{alg:online_agnostic} for the pseudo-code).
The booster has black-box oracle access to two types of auxiliary algorithms: a weak learner, and an online-convex optimizer. 
The booster maintains 
 $N$ instances $\W_1,\ldots, \W_N$ of a weak learning algorithm. 
 Specifically, each weak learner $\W_i$ is a $(\gamma, T)$-\textbf{AWOL} 
 (see Definition \ref{online_agnostic_wl}). 
 The online-convex optimizer is a $([-1,1], N)$-OCO algorithm $\mathcal{A}$ (see Equation \ref{oco_def} below).

\begin{algorithm}[H]
\caption{Online Agnostic Boosting with OCO}
\begin{algorithmic}[1] \label{alg:online_agnostic}
\FOR{$t = 1, \ldots, T$}
\STATE Get $x_t$, predict: $\hat{y}_t = \Pi\big(\frac{1}{\gamma N}\sum_{i=1}^N \W_i(x_t) \big)$. 
\FOR{$i = 1, \ldots, N$}
\STATE If $i > 1$, set $p^i_t = \mathcal{A}(\ell^{1}_t,...,\ell^{i-1}_t)$. Else, set $p_t^1 = 0$.
\STATE Set next loss: $\ell^{i}_t(p) = p(\frac{1}{\gamma} \W_i(x_t)y_t - 1)$.
\STATE Pass $(x_t, y_t^i)$ to $\W_i$, where $y_t^i$ is a random label s.t.\ $\Pr[y_t^i = y_t]=\frac{1+p^i_t}{2}$.
\ENDFOR
\ENDFOR
\end{algorithmic}
\end{algorithm}
\vskip -0.2in
\captionof{figure}{ The algorithm is given oracle access to $N$ instances of a $(\gamma, T)$-AOWL algorithm, $\W_1,...,\W_N$  (see Definition \ref{online_agnostic_wl}), and to a $([-1,1], N)$-OCO algorithm $\mathcal{A}$ (see Equation \ref{oco_def}). The prediction ``$\Pi(\frac{1}{\gamma N}\sum_{i=1}^N \W_i(x_t))$'' in line 2 is a randomized majority-vote, as defined in Equation \ref{proj}.} 

\paragraph{Online Convex Optimization} \hspace{-0.35cm} (see e.g.~\cite{hazan2016introduction}). Recall that in the Online Convex Optimization (OCO) framework, an online player iteratively makes decisions from a  compact convex set $\mathcal{K} \subset \mathbb{R}^d$. At iteration $i=1,...,N$, the online player chooses $p^i \in \K$, and the adversary reveals the cost $\ell^{i}$, chosen from a family $\mathcal{F}$ of bounded convex functions  over $\K$. 
We will refer to an algorithm in this setting as a $(\mathcal{K}, N)$-OCO.
Let $\mathcal{A}$ be a $(\mathcal{K}, N)$-OCO. 
The regret of $\mathcal{A}$ is defined by:
\begin{equation}\label{oco_def}
    R_{\mathcal{A}}(N) = \sum_{i=1}^N \ell^{i}(p^{i}) - \underset{p \in \mathcal{K}}{\min} \sum_{i=1}^N \ell^{i}(p).
\end{equation}

\paragraph{Randomized Majority-Vote/Projection.}
The last component needed to describe our boosting algorithm is the randomized projection ``$\Pi$''
    which is used to predict in Line 2.
    For any $z \in \mathbb{R}$, denote by $\Pi(z)$ the following random label:
\begin{equation}\label{proj}
    \Pi(z) = \begin{cases} 
 \sign(z) & \text{if } |z|\ge 1 \\
 +1 & \text{w.p. } \frac{1+z}{2} \\
 -1 & \text{w.p. } \frac{1-z}{2} 
\end{cases}
\end{equation}

We now state and prove the regret bound for Algorithm~\ref{alg:online_agnostic}.
\begin{proposition}[Regret Bound] \label{thm:oab}
The accumulated gain of Algorithm~\ref{alg:online_agnostic} satisfies:
$$
\frac{1}{T}\text{ }\E\Bigg[\underset{h^* \in \H}{\max }  \text{ }\sum_{t=1}^T
h^*(x_t) y_t -\sum_{t=1}^T
\hat{y}_t y_t \Bigg] \le  \frac{R_{\W}(T)}{\gamma T} + \frac{R_{\A}(N)}{N} ,$$
where $(x_t,y_t)$'s are the observed examples, 
$\hat y_t$'s are the predictions, the expectation is with respect to the algorithm and learners' randomness, and $R_{\W}$ and $R_{\mathcal{A}}$ are the regret terms of the weak learner and the OCO, respectively.
\end{proposition} 

\begin{proof}
The proof follows by combining upper and lower bounds on the expected sum of losses incurred by the OCO algorithm. The bounds follow directly from the weak learning assumption (lower bound) and the OCO guarantee (upper bound). These bounds involve some simple algebraic manipulations. 
It is convenient to abstract out some of these calculations into lemmas, which are described later in this section. \\

Before delving into the analysis, we first clarify several assumptions used below. For simplicity of presentation we assume an oblivious adversary, however, using a standard reduction, our results can be generalized to an adaptive one \footnote{See discussion in \cite{cesa2006prediction}, Pg. 69, as well as Exercise 4.1 formulating the reduction.}. 
Let $(x_1,y_1),...,(x_T,y_T)$ be any sequence of observed examples. Observe that there are several sources of randomness at play; the weak learning algorithm $\W_i$'s internal randomness, the random re-labeling (line 6, Algorithm \ref{alg:online_agnostic}), and the randomized prediction (line 2, Algorithm \ref{alg:online_agnostic}). The analysis below is given in expectation with respect to all these random variables. \\

Note the following fact used in the analysis; for all $i \in [N], t \in [T]$, the random variables $\W_i(x_t)$ and $y_t^i$ are conditionally independent given $p_t^i$ and $y_t$. Since $\E[y_t^i| p_t^i,y_t] = p_t^i \cdot y_t$, using the conditional independence, it  follows that $\E[\W_i(x_t)y_t^i] = \E[\W_i(x_t)p_t^i y_t]$ (see Lemma \ref{lem:expectation} in the Appendix). We can now begin the analysis, starting with lower bounding the expected sum of losses, using the weak learning guarantee, 
\begin{align*} 
   \frac{1}{\gamma}\E\Big[\sum_{i=1}^N \sum_{t=1}^T  \W_i(x_t)\cdot y_t p^i_t\Big]&=  \frac{1}{\gamma}\sum_{i=1}^N \E\Big[\sum_{t=1}^T  \W_i(x_t)\cdot y_t p^i_t\Big] =  \frac{1}{\gamma}\sum_{i=1}^N \E\Big[\sum_{t=1}^T  \W_i(x_t)y_t^i\Big]   \tag{See  Lemma \ref{lem:expectation}}\\
    &\ge  \frac{1}{\gamma}\sum_{i=1}^N \big(\gamma \text{ } \underset{h \in \H}{\max }  \text{ }\E\Big[ \sum_{t=1}^T h(x_t) y_t^i\Big] - R_{\W}(T) \big)
  \tag{Weak Learning  (\ref{online_agnostic_wl})} \\
    &\ge \sum_{i=1}^N \big(\underset{h \in \H}{\max }  \text{ } \sum_{t=1}^T h(x_t) \cdot \E[y_t^i] - \frac{1}{\gamma} R_{\W}(T) \big) \\
    &\ge \sum_{i=1}^N \sum_{t=1}^T h^*(x_t) \cdot \E[y_tp_t^i] - \frac{N}{\gamma} R_{\W}(T) 
  \\
    &=\sum_{i=1}^N \sum_{t=1}^T \E \big[ h^*(x_t) \cdot y_tp_t^i\big] - \frac{N}{\gamma} R_{\W}(T),   
\end{align*}
where $h^*$ is an optimal expert in hindsight for the observed sequence of examples $(x_t, y_t)$'s. Thus, we obtain the lower bound on the expected sum of losses $\sum_t \sum_i \ell^i_t(p^i_t)$ (see Line 5 in Algorithm~\ref{alg:online_agnostic} for the definition of the $\ell^i_t$'s), given by,
\begin{align*} 
\E[\sum_{t=1}^T \sum_{i=1}^N \ell^i_t(p^i_t)] &\ge \sum_{i=1}^N \sum_{t=1}^T \E\big[ p_t^i( h^*(x_t) y_t - 1)\big] - \frac{N}{\gamma} R_{\W}(T) \\
&\ge N \sum_{t=1}^T (h^*(x_t)y_t - 1) - \frac{N}{\gamma} R_{\W}(T) \tag{See Lemma \ref{prop_1} below}.
\end{align*}

For the upper bound, observe that the OCO regret guarantee implies that for any $t \in [T]$, and any $p^*_t \in [-1,1]$, 
$$
    \E\Big[\frac{1}{N} \sum_{i=1}^N \ell^i_t(p^i_t)\Big] \le p^*_t \bigg(\bigg( \frac{1}{\gamma N}  \sum_{i=1}^N \E\big[\W_i(x_t)\big] \bigg)y_t -1 \bigg)   +  \frac{1}{N}R_{\mathcal{A}}(N),
$$
 Thus, by setting $p_t^*$ according to Lemma \ref{proj_lemma} (see below, with $\hat{h}(x) := \frac{1}{\gamma N}  \sum_{i=1}^N \E\big[\W_i(x)\big]$), and summing over $t \in [T]$, we get,
$$
\E\Big[\frac{1}{N}\sum_{t=1}^T \sum_{i=1}^N \ell^i_t(p^i_t)\Big] \le \sum_{t=1}^T (\E[\hat{y}_t] y_t - 1) + \frac{T}{N}R_{\mathcal{A}}(N).
$$
By combining the lower and upper bounds for $\E\big[\frac{1}{NT}\sum_t\sum_i \ell^i_t(p^i_t)\big]$, we get,
$$
\frac{1}{T}\sum_{t=1}^T
\E[\hat{y}_t] y_t \ge  \frac{1}{T}\sum_{t=1}^T
h^*(x_t) y_t  -\frac{R_\W(T)}{\gamma T} - \frac{R_\A(N)}{N}.
$$

\end{proof}

It remains to prove two Lemmas that are used in the proof of the theorem above, as well as in the more general settings in the following sections.
 \begin{lemma} \label{prop_1}
For any $p \in [-1,1]$, an example pair $(x,y)$, and $h: \X \rightarrow \{-1,1\}$, we have:
$$
p(h(x) y - 1)\ge h(x) y - 1.
$$
 \end{lemma}
 \begin{proof}
    Let $z = h(x) y - 1$. Observe that $z \in \{-2,0\}$. Thus, since $p \in [-1,1]$, $pz \ge z$. 
    \end{proof}
    
\begin{lemma} \label{proj_lemma}  Given an example pair $(x,y)$, and $\hat{h}: \X \rightarrow \mathbb{R}$, there exists $p^* \in \{0,1\}$, such that,
$$
p^* ( \hat{h}(x) y - 1)\le \hat{y} y - 1,
$$
where $\hat{y}_t = \E[\Pi(\hat{h}(x))]$, with expectation taken only w.r.t. the randomness of $\Pi$ (see Definition (\ref{proj})).
\end{lemma}
\begin{proof}
If $|\hat{h}(x)|\le 1$, $\widehat{y} = \hat{h}(x)$ and by setting $p^* = 1$, the equality follows.
Thus, assume $|\hat{h}(x)| > 1$, and consider the following cases:
\begin{itemize}
    \item If $\hat{h}(x) y - 1 > 0$, then $\widehat{y} y - 1 = 0$. Hence, by setting $p^* = 0$, the equality follows. 
    \item If $\hat{h}(x) y - 1 < 0$, then since $|\hat{h}(x)| > 1$ it must be that $\text{sign}(\hat{h}(x))y = -1$, and $\widehat{y} y - 1 = -2$. Since $|\hat{h}(x)| > 1$, we have $\hat{h}(x) y - 1 \le -2$. Hence, by setting $p^* = 1$ the inequality holds. 
\end{itemize}
\end{proof}

\subsection{Proof of Theorem~\ref{thm:main}}\label{sec:proofmain}

The proof of Theorem \ref{thm:main} is a direct corollary of Proposition \ref{thm:oab}, by plugging \textit{Online Gradient Descent (OGD)} to be the OCO algorithm $\mathcal{A}$ (e.g., see \cite{hazan2016introduction} Chapter 3.1):
the OGD regret is $O(GD\sqrt{N})$, where $N$ is the number of iterations, $G$ is an upper bound on the gradient of the losses, 
and $D$ is the diameter of the set $\mathcal{K} = [-1,1]$. In our setting, $G \le \frac{2}{\gamma}$, and $D = 2$. Hence, $R_{\mathcal{A}}=O(\sqrt{N}/\gamma)$, and the overall bound on the regret follows. 

\section{Statistical Boosting via Improper Game Playing}\label{sec:general}
In this section we first give a game-theoretic perspective of our method when applied to the statistical setting (Subsection \ref{zero_sum_games}). We then demonstrate a general reduction from both the agnostic (Subsection \ref{subsec:ab}), and realizable (Subsection \ref{subsec:rb}) boosting settings, to online convex optimization. The following algorithm is given as input a sample $S = (x_1,y_1),\ldots, (x_m,y_m) \in \X \times \Y$, and has a black-box access to two auxiliary algorithms: a weak learner, and an online-convex optimizer. 
    Note that this in fact defines a family of boosting algorithms, depending on the choice of the online-convex optimizer.

\begin{algorithm}[H]
\caption{Boosting with OCO}
\begin{algorithmic}[1] \label{gen_boost_alg}
\FOR{$t = 1, \ldots, T$}
\STATE Pass $m_0$ examples to $\W$ drawn from the following distribution:
\STATE \textbf{Realizable}: Draw $(x_i,y_i)$ w.p. $\propto p_t(i)$\footnotemark.\STATE \textbf{Agnostic}: Draw $x_i$ w.p. $\frac{1}{m}$, and re-label according to $y_i p_t(i)$.
\STATE Let $h_t$ be the weak hypothesis returned by $\W$. 
\STATE Set loss: $\ell_t(p) =  \sum_{i=1}^m p(i) (\frac{1}{\gamma}h_t(x_i)y_i - 1)$.
\STATE Update: $p_{t+1} = \mathcal{A}(\ell_1,...,\ell_t)$.
\ENDFOR
\RETURN $\bar{h}(x) =  \Pi\big(\frac{1}{\gamma T}\sum_{t=1}^T h_t(x) \big)$. 
\end{algorithmic}
\end{algorithm}
\vskip -0.2in
\captionof{figure}{ The algorithm has oracle access to either a $(\gamma, \epsilon_0, m_0)$-AWL algorithm (see Definition \ref{agnostic_wl}) 
    or a $(\gamma, m_0)$-WL algorithm (see Definition \ref{vanilla_wl}). Both are denoted as $\W$. 
 The optimizer is a $(\gamma, \mathcal{K}, T)$-OCO algorithm $\mathcal{A}$ (see Definition \ref{oco_def}), 
    where $\mathcal{K} = [0,1]^m$ in the realizable case and $\mathcal{K} = [-1,1]^m$ in the agnostic case. In line 4, we pass $(x_i, y_t^i)$ to $\W_i$, where $y_t^i$ is a random label s.t.\ $\Pr[y_i^t = y_i]=\frac{1+p_t(i)}{2}$.
    The final hypothesis ``$\Pi\big(\frac{1}{\gamma T}\sum_{t=1}^T h_t(x) \big)$'' is a randomized majority-vote, as defined in Equation \ref{proj}.
    } 
\vskip 0.2in
\footnotetext{Note that when $p_t = \mathbf{0}$ is constantly zero then the distribution used in the realizable setting
    is not well defined. There are several ways to circumvent it. 
    Concretely, we proceed in such case by setting $h_t=h_{t-1}$ and proceeding to step~6.}

\subsection{Solving Zero Sum Games Improperly Using an Approximate Optimization Oracle} \label{zero_sum_games}

Our framework uses as a main building block a procedure for approximately solving zero sum games using an approximate optimization oracle. It is described in this section. 

In the zero sum games setting, there are two players A and B, and a payoff function $g$ that depends on the players' strategies. Player A's goal is to minimize the payoff, while player B's goal is to maximize it. Let $\K_A$ and $\K_B$ be the convex, compact decision sets of players A and B, respectively, and assume that $g$ is convex-concave. By Sion's minimax theorem \cite{sion}, the value of the game is well-defined, and we denote it by $\lambda^*$:
$$
\min_{p\in \K_A}\max_{q\in\K_B} g(p, q) = \max_{q\in \K_B}\min_{p\in\K_A} g(p, q) = \lambda^*
$$

Let $\K_B'$ be a convex, compact set such that $\K_B \subseteq \K_B'$. We refer to strategies in $\K_B$ as proper strategies, while those in $\K_B'$ are improper strategies. We consider a modified zero sum games setting where the payoff function $g$ is defined on $\K_B'$, the set of improper strategies. Note that $\lambda^*$ is defined with respect to the set of proper strategies, and it is still a well-defined quantity in this game.

\textbf{Assumption 1:} Player B has access to a randomized approximate optimization oracle $\W$. Given any $p\in \K_A$, $\W$ outputs an improper best response: a strategy $q\in\K_B'$ such that $\E[g(p, q)] \ge \max_{q^*\in\K_B} g(p, q^*) - \epsilon_0$, where the expectation is taken over the randomness of $\W$. 

\textbf{Assumption 2:} Player B is allowed to play strategies in $\K_B'$. 

\textbf{Assumption 3:} Player A has access to a possibly randomized $(\mathcal{K}_A, T)$-OCO algorithm $\A$ with regret $R_\A(T)$ (See Definition \ref{oco_def}).

\begin{algorithm}[H]
\caption{Improper Zero Sum Games with Oracles}
\begin{algorithmic}[1] \label{game_algo}
\FOR{$t = 1, \ldots, T$}
\STATE Player A plays $p_t$.
\STATE Player B plays $q_t\in \K_B'$, where $q_t = \W(p_t)$.
\STATE Define loss: $\ell_t(p) = g(p, q_t)$
\STATE Player A updates $p_{t+1} = \mathcal{A}(\ell_1,...,\ell_t)$.
\ENDFOR
\end{algorithmic}
\end{algorithm}

\begin{proposition}\label{game_prop}
If players A and B play according to Algorithm \ref{game_algo}, then player B's average strategy $\bar{q} = \frac{1}{T}\sum_{t=1}^T q_t$, $\bar{q}\in \K_B'$, satisfies for any $p^*\in \K_A$,
$$
\lambda^*\le \E[g(p^*, \bar{q})] + \frac{R_\A(T)}{T} + \epsilon_0,
$$
where the expectation is taken over the randomness of $\W$.
\end{proposition}

\begin{proof}
Since the game is well-defined over $\K_A$ and $\K_B$, there exists a max-min strategy $q^* \in \K_B$ for player B such that for all $p\in \K_A$, $g(p, q^*) \ge \lambda^*$. Let $\bar{p} = \frac{1}{T}\sum_{t=1}^T p_t$, and observe that since the $p_t$'s depend on the sequence of $q_t$'s, they are also random variables, as well as $\bar{p}$. We have,
\begin{align*}
    \E[\frac{1}{T}\sum_{t=1}^T g(p_t, q_t)] \ge  \E[\frac{1}{T}\sum_{t=1}^T g(p_t, q^*)] - \epsilon_0 \ge \E[g(\bar{p}, q^*)] - \epsilon_0   \ge \lambda^* - \epsilon_0  .
\end{align*}
The first inequality is due to Assumption 1, where $\E[g(p_t, q_t)] \ge \max_{q\in\K_B} g(p_t, q) - \epsilon_0  \ge g(p_t, q^*)- \epsilon_0  .$ The second inequality holds because $g$ is convex in $p$. 

Now, let $\bar{q} = \frac{1}{T}\sum_{t=1}^T q_t$; note that $\bar{q}\in \K_B'$ since $\K_B'$ is convex. 
For the upper bound, observe that the OCO regret guarantee implies that for any $p^* \in \K_A$ we have,
\begin{align*}
    \E[\frac{1}{T}\sum_{t=1}^T g(p_t, q_t)] \le  \E[\frac{1}{T} \sum_{t=1}^T g(p^*, q_t)] + \frac{R_\A(T)}{T}\le \E[g(p^*, \bar{q})] + \frac{R_\A(T)}{T},
\end{align*}
where the second inequality holds because $g$ is concave in $q$. Combining the lower and upper bounds yields the theorem.
\end{proof}


\subsection{Statistical Agnostic Boosting} \label{subsec:ab}

We will use the following notation. 
Let $D$ be a distribution over $\X\times \Y$ and let $h:\X\to\Y$ be an hypothesis.
Define the correlation of $h$ with respect to $D$ by:
\[\cor_D(h) = \E_{(x,y)\sim D}[h(x)\cdot y].\]

\begin{definition}[Empirical Agnostic Weak Learning Assumption] \label{agnostic_wl}
Let $\H\subseteq \{\pm 1\}^{\X}$ be a hypothesis class and let $\x=(x_1\ldots x_m)\in \X$ denote an unlabeled sample.
    A learning algorithm $\W$ is a $(\gamma, \epsilon_0, m_0)$-\textbf{agnostic weak learner (AWL)} for~$\H$ with respect to~$\x$
    if for any labels $\y= (y_1,\ldots, y_m)$, 
    \[\E_{S'}[\cor_{\mu\times \y}(\W(S'))]\geq \gamma \underset{h^* \in \H}{\max}  \cor_{\mu\times \y}(h^*) - \epsilon_0,\]
    where $\mu\times \y$ is the distribution which uniformly assigns to each example $(x_i,y_i)$ probability $1/m$,
    and $S'$ is an independent sample of size $m_0$ drawn from $\mu\times \y$.
%
\end{definition}
In accordance with previous works, we focus on the setting  where $\gamma$ is a small constant (say $\gamma=0.1$) 
    and $\eps_0 \approx d/\sqrt{m}$, where $d$ is the VC-dimension of $\H$ (see~\cite{kanade2009potential} for a detailed discussion).
    We stress however that our results apply for any setting of $\gamma,\epsilon_0\in [0,1]$. 

\vspace{2mm}

The above weak learning assumption can be seen as an empirical variant of the assumption in \cite{kanade2009potential},
where $\mu$ is replaced with the population distribution over $X$ and the labels $y_i$'s are replaced with an arbitrary classifier $c:X\to\{\pm 1\}$. Both of these assumptions are weaker than the standard agnostic weak learning assumption,
for which the guarantee holds with respect to every distribution $D$ over $\X\times\{\pm 1\}$.
It will be interesting to investigate the relationship between the assumption of \cite{kanade2009potential}
and our empirical variant, however this is beyond the scope of this work.
   
We now state and prove the regret bound for Algorithm~\ref{gen_boost_alg}.
\begin{theorem}[Empirical Agnostic Boosting] \label{thm:offline}
The correlation of the output  of Algorithm~\ref{gen_boost_alg}, which is denoted $\bar{h}$, satisfies:
\begin{equation}\label{eq:emp_cor}
\E\big[\cor_S(\bar{h})\big] \ge  \underset{h^* \in \H}{\max }  \text{ } \E\big[\cor_S(h^*)\big]  - \Biggl(\frac{\epsilon_0}{\gamma} + O(\frac{1}{\gamma \sqrt{T}})\Biggr).
\end{equation}
\end{theorem} 
\paragraph{Generalization.}
The above theorem asserts that the correlation of the output hypothesis is competitive 
with the best hypothesis in $\H$ with respect to the \underline{empirical} distribution. 
Obtaining a similar guarantee with respect to the \underline{population} distribution
can be obtained using standard arguments. One way of deriving it is via a sample compression argument
(which is natural in boosting; see, e.g.~\cite{Schapire2012,DavidMY16}):
indeed, the final hypothesis $\bar h$ is obtained by aggregating the $T$ weak hypotheses $h_t$'s,
each of which is determined by the $m_0$ examples fed to the weak learner.
Thus, $\bar h$ can be encoded by $T\cdot m_0$ input examples
and hence the entire algorithm forms a sample compression scheme of this size.
Consequently, by setting the input sample $m = \tilde O(T\cdot m_0/\eps^2)$
we get the same guarantee like in Equation~\ref{eq:emp_cor} up to an additive error of $\eps$.
\vspace{2mm}
\begin{proof}[Theorem~\ref{thm:offline}]
The proof  has two parts. 
The first part is a straightforward reduction to the game-theoretic setup of Proposition~\ref{game_prop},
and the second part shows how to project the ``improper'' strategy
obtained by Proposition~\ref{game_prop} to the desired output hypothesis.

\vspace{2mm}

\textit{Reduction to Proposition~\ref{game_prop}.}
The agnostic version of Algorithm~\ref{gen_boost_alg} can be presented as an instance of Algorithm~\ref{game_algo},
where Player A and B are the weak learner and the OCO oracle algorithms, respectively. The decision sets are
$\K_A = [-1, 1]^m$, $\K_B = \Delta_\H$, and $\K_B' = \frac{1}{\gamma}\Delta_{\H}$,
and the payoff function $g(\cdot,\cdot)$ is given by
\[g(p, q) = \sum_{i=1}^m p(i)(q(x_i)y_i - 1),\]
where $p\in \K_A$ is a vector in the $m$ dimensional continuous cube, 
and $q\in \K_B'$ is a non-negative combination of hypotheses in $\H$ 
(and so $q$ corresponds to the mapping $x\mapsto \sum_{h\in \H}q(h)\cdot h(x)$).  
We leave it to the reader to verify that the agnostic weak learner corresponds to an approximate optimization oracle $\W$.
Namely, for any $p\in \K_A$ the output $q'=\W(p)$ satisfies $q'\in \K_B'$ and
\[
\E[g(p, q')] \ge \underset{q\in\K_B}{\max}\text{ } g(p, q) - \frac{\epsilon_0 m}{\gamma}.
\]


Furthermore, it can be shown that the value of the above game is 
$$
\lambda^* = m\cdot \underset{h\in\H}{\max}\text{ } \cor_{S} (h) - m.
$$
This can be done by (i) observing that the strategy $p=(1,1,\ldots 1)\in \K_A$ is dominant for Player $A$
and (ii) computing $\max_{q\in \K_B} g(p,q)$ which is equal to $\lambda^*$ (since $p$ is dominating).

Now, Proposition \ref{game_prop} implies that for any $p\in [-1, 1]^{m}$, we have 
\begin{equation}\label{eq: 3}
m\cdot \underset{h\in\H}{\max}\text{ } \cor_{S} (h) - m \le \E\big[\sum_{i=1}^{m} p(i)(\bar{q}(x_i)y_i - 1)\big] + \frac{R_\A(T)}{T} + \frac{\epsilon_0 m}{\gamma},
\end{equation}
where $\bar{q}(x_i) = \frac{1}{\gamma T}\sum_{t=1} h_t(x_i) \in \K_B'$. 

\vspace{2mm}

\textit{Projection.}
Recall that the output hypothesis $\bar h$ is defined using the projection $\Pi$ (see Definition \ref{proj}):
\[\bar{h}(x_i) = \Pi(\bar{q}(x_i)).\] 
Now, by Lemma~\ref{proj_lemma} there exists $p^*$ such that 
\begin{align*}
m\cdot \underset{h\in\H}{\max}\text{ } \cor_{S} (h) - m &\le \E\big[\sum_{i=1}^{m} p^*(i)(\bar{q}(x_i)y_i - 1)\big] + \frac{R_\A(T)}{T} + \frac{\epsilon_0 m}{\gamma}\tag{Equation~\ref{eq: 3}}\\
&\le m\cdot  \E[\cor_S(\bar{h})] - m+ \frac{R_\A(T)}{T} + \frac{\epsilon_0 m}{\gamma} &\tag{Lemma~\ref{proj_lemma}}
\end{align*}
where the expectation is taken over the randomness of the projection, the weak learner, and the random samples given to the weak learner.
Simple manipulation on the above inequality directly yields  
\[
\underset{h\in\H}{\max}\text{ } \cor_{S} (h) \le \E[\cor_{S} (\bar{h})] + \frac{R_\A(T)}{T m} + \frac{\epsilon_0}{\gamma}.
\]
 If we use OGD as the OCO algorithm, we have $R_\A(T) = GD\sqrt{T}$, where $G \le  \frac{2\sqrt{m}}{\gamma}$ and $D =2\sqrt{m}$. 
We arrive at the theorem by plugging in $\frac{R_\A(T)}{Tm}.$


\end{proof}

\subsection{Statistical Realizable Boosting} \label{subsec:rb}

\begin{definition}[Empirical Weak Learning Assumption~\cite{Schapire2012}] \label{vanilla_wl} 
Let $\H\subseteq \{\pm 1\}^{\X}$ be a hypothesis class, and let $S=\{(x_1, y_1),\ldots,( x_m, y_m)\}\in \X\times \{\pm 1\}$ be a sample. A learning algorithm $\W$ is a $(\gamma, m_0)$-\textbf{weak learner (WL)} for~$\H$ with respect to~$S$
    if for any distribution $\p= (p_1,\ldots, p_m)$ which assigns each example $(x_i, y_i)$ with probability $p_i$,
    $$\E_{S'}[\cor_{\p}(\W(S'))]\geq \gamma,$$
    where $S'$ is an independent sample of size $m_0$ drawn from $\p$.

\end{definition}
\begin{theorem}\label{thm:statistical_realizable}
The correlation of the output of Algorithm \ref{gen_boost_alg}, denoted $\bar{h}$, satisfies 
$$
\E[\cor_S(\bar{h})]\ge 1 - O\Big(\frac{1}{\gamma\sqrt{T}}\Big).
$$
\end{theorem}

The proof follows in a similar structure as in Theorem \ref{thm:offline}, and is deffered to the Appendix.

\section{Online Realizable Boosting} \label{sec:online_realizable}
In this section, we give an online realizable boosting algorithm, and state the regret bound. The result is along similar lines as our main result given in Section \ref{sec:online}. We first state the weak learning assumption for the online realizable setting. 

\begin{definition}[Online Weak Learning] \label{online_wl}
Let $\H\subseteq \{\pm 1\}^{\X}$ be a class of experts, let $T$ denote the horizon length, and let $\gamma>0$ denote the advantage.
    An online learning algorithm $\W$ is a $(\gamma, T)$-\textbf{weak online learner (WOL)} for~$\H$ 
    if for any sequence $(x_1, {y}_1), ...,(x_T, {y}_T)\in \X \times \{\pm 1\}$ that is realizable by $\H$, at every iteration $t \in [T]$,
    the algorithm outputs $\W(x_t) \in \{\pm 1\}$ such that, 
\[
\sum_{t=1}^T \E\Bigl[\W(x_t)\Bigr] {y}_t \ge \gamma T - R_{\W}(T),
\]
where the expectation is taken over the randomness of the weak learner $\W$ and $R_{\W}: \mathbb{N} \rightarrow \R_+$ 
is the additive regret: a non-decreasing, sub-linear function of $T$. 
\end{definition}

Similar to the online agnostic case, the boosting algorithm is given access to $N$ instances of a $(\gamma, T)$-WOL algorithm  (see Definition \ref{online_wl}) and a $(\mathcal{K}, T)$-OCO algorithm $\mathcal{A}$ (see Definition \ref{oco_def}). Instead of setting $\mathcal{K} = [-1,1]$ as in the agnostic case, we set $\mathcal{K} = [0,1]$. The algorithm for online boosting is exactly the same as in the agnostic online case (see Algorithm \ref{alg:online_agnostic}), except for line 6. In the online agnostic case, we pass a relabeled data point to $\W_i$,  while the algorithm below does not relabel the data points.

\begin{algorithm}
\caption{Online Boosting with OCO}
\begin{algorithmic}[1] \label{gen_ob_alg}
\FOR{$t = 1, \ldots, T$}
\STATE Get $x_t$, predict: $\hat{y}_t = \Pi\big(\frac{1}{\gamma N}\sum_{i=1}^N \W_i(x_t) \big)$.
\FOR{$i = 1, \ldots, N$}
\STATE If $i > 1$, set $p^i_t = \mathcal{A}(\ell^{1}_t,...,\ell^{i-1}_t)$. Else, set $p_t^1 = 1/2$.
\STATE Set next loss: $\ell^{i}_t(p) = p(\frac{1}{\gamma} \W_i(x_t)y_t - 1)$.
\STATE Pass $(x_t, y_t)$ to $\W_i$ w.p. $p^i_t$.
\ENDFOR
\ENDFOR
\end{algorithmic}
\end{algorithm}

\vskip 0.1in
\noindent The following theorem proves the realizable \textit{online} boosting result. Observe that in the realizable case, $\max_{h\in H} cor_S(h) = 1$. Let $\Tilde{R}_{\W}(T) := 2R_{\W}(T) + \Tilde{O}(\sqrt{T})$. Note that the error can be made arbitrarily small, by setting the number of weak learners to $N = O(\frac{1}{\gamma^2 \epsilon^2})$ and the number of iterations of Algorithm \ref{gen_ob_alg}  to $T = O(\frac{1}{\gamma^2 \epsilon^2})$, for any $\epsilon > 0$. Thus, for an OCO algorithm (\ref{oco_def}) with regret bound $R_{\mathcal{A}}(N, G_{\gamma}) = O(\frac{\sqrt{N}}{\gamma})$, and a weak learner with regret bound $R_{\W}(T) = O(\sqrt{T})$, by the following theorem, we get that the online correlation of the booster is at least $\cor_S(h^*) - \epsilon$. 

\begin{theorem} \label{thm:online_realizable} 
The accumulated gain of Algorithm~\ref{gen_ob_alg} satisfies:
$$\frac{1}{T}\sum_{t=1}^T \E[\hat{y}_t y_t] \ge  1 -\Biggl(\frac{\Tilde{R}_\W(T)}{\gamma T} + \frac{R_\A(N)}{N}\Biggr).$$
where $(x_t,y_t)$'s are the observed examples, 
$\hat y_t$'s are the predictions, the expectation is with respect to the algorithm and learners' randomness, $\Tilde{R}_{\W}(T) := 2R_{\W}(T) + \Tilde{O}(\sqrt{T})$, and $R_{\W}$ and $R_{\mathcal{A}}$ are the regret terms of the weak learner and the OCO, respectively.

\end{theorem}
The proof follows similarly to the proof of Proposition \ref{thm:oab}, and is deferred to the Appendix.

\section{Discussion}

We have presented the first boosting algorithm for agnostic online learning. In contrast to the realizable setting, we do not place any restrictions on the online sequence of examples. It remains open to prove lower bounds on online agnostic boosting as a function of the natural parameters of the problem and/or improve our upper bounds. 

\bibliography{bib}

\begin{thebibliography}{10}

\bibitem{Abernathy17onlilnedp}
Jacob~D. Abernethy, Chansoo Lee, Audra McMillan, and Ambuj Tewari.
\newblock Online learning via differential privacy.
\newblock {\em CoRR}, abs/1711.10019, 2017.

\bibitem{agarwal2019boosting}
Naman Agarwal, Nataly Brukhim, Elad Hazan, and Zhou Lu.
\newblock Boosting for dynamical systems.
\newblock {\em arXiv preprint arXiv:1906.08720}, 2019.

\bibitem{agarwal2019online}
Naman Agarwal, Brian Bullins, Elad Hazan, Sham Kakade, and Karan Singh.
\newblock Online control with adversarial disturbances.
\newblock In {\em International Conference on Machine Learning}, pages
  111--119, 2019.

\bibitem{agarwal2019logarithmic}
Naman Agarwal, Elad Hazan, and Karan Singh.
\newblock Logarithmic regret for online control.
\newblock In {\em Advances in Neural Information Processing Systems 32}, pages
  10175--10184. 2019.

\bibitem{Agarwal17dponline}
Naman Agarwal and Karan Singh.
\newblock The price of differential privacy for online learning.
\newblock In Doina Precup and Yee~Whye Teh, editors, {\em Proceedings of the
  34th International Conference on Machine Learning, {ICML} 2017, Sydney, NSW,
  Australia, 6-11 August 2017}, volume~70 of {\em Proceedings of Machine
  Learning Research}, pages 32--40. {PMLR}, 2017.

\bibitem{AlonLMM19}
Noga Alon, Roi Livni, Maryanthe Malliaris, and Shay Moran.
\newblock Private {PAC} learning implies finite {L}ittlestone dimension.
\newblock In {\em Proceedings of the 51st Annual ACM Symposium on the Theory of
  Computing}, STOC '19, New York, NY, USA, 2019. ACM.

\bibitem{bendavid2001agnostic}
Shai Ben-David, Philip~M Long, and Yishay Mansour.
\newblock Agnostic boosting.
\newblock In {\em International Conference on Computational Learning Theory},
  pages 507--516. Springer, 2001.

\bibitem{gradientboosting}
Alina Beygelzimer, Elad Hazan, Satyen Kale, and Haipeng Luo.
\newblock Online gradient boosting.
\newblock In C.~Cortes, N.~D. Lawrence, D.~D. Lee, M.~Sugiyama, and R.~Garnett,
  editors, {\em Advances in Neural Information Processing Systems 28}, pages
  2458--2466. Curran Associates, Inc., 2015.

\bibitem{beygelzimer2015optimal}
Alina Beygelzimer, Satyen Kale, and Haipeng Luo.
\newblock Optimal and adaptive algorithms for online boosting.
\newblock In {\em International Conference on Machine Learning}, pages
  2323--2331, 2015.

\bibitem{bousquet2019passing}
Olivier Bousquet, Roi Livni, and Shay Moran.
\newblock Passing tests without memorizing: Two models for fooling
  discriminators, 2019.

\bibitem{Bun20Privateonline}
Mark Bun, Roi Livni, and Shay Moran.
\newblock {An Equivalence Between Private Classification and Online
  Prediction}.

\bibitem{cesa2006prediction}
Nicolo Cesa-Bianchi and G{\'a}bor Lugosi.
\newblock {\em Prediction, learning, and games}.
\newblock Cambridge university press, 2006.

\bibitem{chen2016communication}
Shang-Tse Chen, Maria-Florina Balcan, and Duen~Horng Chau.
\newblock Communication efficient distributed agnostic boosting.
\newblock In {\em Artificial Intelligence and Statistics}, pages 1299--1307,
  2016.

\bibitem{chenonline}
Shang-Tse Chen, Hsuan-Tien Lin, and Chi-Jen Lu.
\newblock An online boosting algorithm with theoretical justifications, 2012.

\bibitem{DavidMY16}
Ofir David, Shay Moran, and Amir Yehudayoff.
\newblock Supervised learning through the lens of compression.
\newblock In Daniel~D. Lee, Masashi Sugiyama, Ulrike von Luxburg, Isabelle
  Guyon, and Roman Garnett, editors, {\em Advances in Neural Information
  Processing Systems 29: Annual Conference on Neural Information Processing
  Systems 2016, December 5-10, 2016, Barcelona, Spain}, pages 2784--2792, 2016.

\bibitem{feldman2009distribution}
Vitaly Feldman.
\newblock Distribution-specific agnostic boosting.
\newblock {\em arXiv preprint arXiv:0909.2927}, 2009.

\bibitem{Freund90majority}
Yoav Freund.
\newblock Boosting a weak learning algorithm by majority.
\newblock In Mark~A. Fulk and John Case, editors, {\em Proceedings of the Third
  Annual Workshop on Computational Learning Theory, {COLT} 1990, University of
  Rochester, Rochester, NY, USA, August 6-8, 1990}, pages 202--216. Morgan
  Kaufmann, 1990.

\bibitem{freund1996game}
Yoav Freund and Robert~E Schapire.
\newblock Game theory, on-line prediction and boosting.
\newblock In {\em Proceedings of the ninth annual conference on Computational
  learning theory}, pages 325--332, 1996.

\bibitem{Freund97decision}
Yoav Freund and Robert~E. Schapire.
\newblock A decision-theoretic generalization of on-line learning and an
  application to boosting.
\newblock {\em J. Comput. Syst. Sci.}, 55(1):119--139, 1997.

\bibitem{freund1999adaptive}
Yoav Freund and Robert~E Schapire.
\newblock Adaptive game playing using multiplicative weights.
\newblock {\em Games and Economic Behavior}, 29(1-2):79--103, 1999.

\bibitem{gavinsky2003optimally}
Dmitry Gavinsky.
\newblock Optimally-smooth adaptive boosting and application to agnostic
  learning.
\newblock {\em Journal of Machine Learning Research}, 4(May):101--117, 2003.

\bibitem{privacy}
Alon Gonen, Elad Hazan, and Shay Moran.
\newblock Private learning implies online learning: An efficient reduction.
\newblock {\em CoRR}, abs/1905.11311, 2019.

\bibitem{hazan2016introduction}
Elad Hazan.
\newblock Introduction to online convex optimization.
\newblock {\em Foundations and Trends{\textregistered} in Optimization},
  2(3-4):157--325, 2016.

\bibitem{Hazan2019TheNC}
Elad Hazan, Sham~M. Kakade, and Karan Singh.
\newblock The nonstochastic control problem.
\newblock {\em arXiv preprint arXiv:1911.12178}, 2019.

\bibitem{Joseph2019TheRO}
Matthew Joseph, Jieming Mao, Seth Neel, and Aaron Roth.
\newblock The role of interactivity in local differential privacy.
\newblock In {\em FOCS}, 2019.

\bibitem{kalai2008agnostic}
Adam~Tauman Kalai, Yishay Mansour, and Elad Verbin.
\newblock On agnostic boosting and parity learning.
\newblock In {\em Proceedings of the fortieth annual ACM symposium on Theory of
  computing}, pages 629--638, 2008.

\bibitem{KalaiS05}
Adam~Tauman Kalai and Rocco~A. Servedio.
\newblock Boosting in the presence of noise.
\newblock {\em J. Comput. Syst. Sci.}, 71(3):266--290, 2005.

\bibitem{kanade2009potential}
Varun Kanade and Adam Kalai.
\newblock Potential-based agnostic boosting.
\newblock In {\em Advances in neural information processing systems}, 2009.

\bibitem{Kearns88unpublished}
M.~Kearns.
\newblock Thoughts on hypothesis boosting.
\newblock Unpublished, December 1988.

\bibitem{LongS08}
Philip~M. Long and Rocco~A. Servedio.
\newblock Adaptive martingale boosting.
\newblock In Daphne Koller, Dale Schuurmans, Yoshua Bengio, and L{\'{e}}on
  Bottou, editors, {\em Advances in Neural Information Processing Systems 21,
  Proceedings of the Twenty-Second Annual Conference on Neural Information
  Processing Systems, Vancouver, British Columbia, Canada, December 8-11,
  2008}, pages 977--984. Curran Associates, Inc., 2008.

\bibitem{MansourM02}
Yishay Mansour and David~A. McAllester.
\newblock Boosting using branching programs.
\newblock {\em J. Comput. Syst. Sci.}, 64(1):103--112, 2002.

\bibitem{Neel2018HowTU}
Seth Neel, Aaron Roth, and Zhiwei~Steven Wu.
\newblock How to use heuristics for differential privacy.
\newblock In {\em {FOCS}}, 2019.

\bibitem{Schapire90boosting}
Robert~E. Schapire.
\newblock The strength of weak learnability.
\newblock {\em Machine Learning}, 5(2):197--227, 1990.

\bibitem{Schapire2012}
Robert~E. Schapire and Yoav Freund.
\newblock {\em {Boosting: Foundations and Algorithms}}.
\newblock Cambridge university press, 2012.

\bibitem{sion}
Maurice Sion.
\newblock On general minimax theorems.
\newblock {\em Pacific J. Math.}, 8(1):171--176, 1958.

\end{thebibliography}
\bibliographystyle{plain}

\appendix
\newpage
\section*{Appendix}

\begin{lemma}\label{lem:expectation}
Let $p^i, \W_i(x), y^i, y$ be random variables, such that $y, y^i \in\{\pm 1\}$, and $\Pr[y^i = y|p^i, y] = \frac{1 + p^i}{2}$, $\Pr[y^i = -y|p^i, y] = \frac{1 - p^i}{2}$. Moreover, $\W_i(x)$ and $y^i$ are conditionally independent given $p^i$ and $y$, namely $\Pr[\W_i(x), y^i|p^i, y] = \Pr[\W_i(x)|p^i, y]\Pr[y^i|p^i, y]$ Then $\E[\W_i(x)\cdot y^i] = \E[\W_i(x)\cdot yp^i]$.
\end{lemma}
\begin{proof}
\begin{align*}
    \E[\W_i(x)\cdot y^i] &= \E_{p^i, y}[\E[\W_i(x)\cdot y^i|p^i, y]] \tag{law of total expectation}\\
    &= \E_{p^i, y}[\E[\W_i(x)|p^i, y]\cdot \E[y^i|p^i, y]] \tag{conditional independence}\\
    &=\E_{p^i, y}[yp^i\cdot \E[\W_i(x)|p^i, y]]\tag{$\E[y^i|p^i, y] = yp^i$}\\
    &= \E[\W_i(x)\cdot yp^i]
\end{align*}
\end{proof}

\subsection*{Proof of Theorem \ref{thm:online_realizable}}
We first state the following Lemma that will be used in the proof:
\begin{lemma} \label{beygelzimer_lemma}
For any weak learner $(\gamma, T)$-WL $\W$, there exists $c = \tilde{O}(\sqrt{\sum_t p_t}) + 2R_1(T)$ such that for any sequence $p_1,...,p_T \in [0,1]$,
$$
\sum_{t=1}^T p_t \cdot \W(x_t) y_t \ge \gamma \sum_{t=1}^T p_t  - c.
$$
\end{lemma}
\begin{proof}
The proof of this lemma is based on the proof of Lemma 1 in \cite{beygelzimer2015optimal}. 
\end{proof}

We are now ready to prove Theorem \ref{thm:online_realizable}.
Let $h^*$ be an optimal hypothesis in hindsight for the given sequence of examples. We prove by lower and upper bounding the sum of losses. 
For simplicity of presentation we assume an oblivious adversary, however, using a standard reduction, our results can be generalized to an adaptive one \footnote{See discussion in \cite{cesa2006prediction}, Pg. 69, as well as Exercise 4.1 formulating the reduction.}. 
Let $(x_1,y_1),...,(x_T,y_T)$ be any sequence of observed examples. Observe that there are several sources of randomness at play; the weak learning algorithm $\W_i$'s internal randomness, the booster randomly passing the example to $W_i$ (line 5, Algorithm \ref{gen_ob_alg}), and the randomized prediction (line 2, Algorithm \ref{gen_ob_alg}). The analysis below is given in expectation with respect to all these random variables. We can now begin the analysis, starting with lower bounding the expected sum of losses, using the weak learning guarantee, 

\begin{align*} 
    \frac{1}{\gamma}\E\Big[\sum_{i=1}^N \sum_{t=1}^T  \W_i(x_t)\cdot y_t p^i_t\Big] 
    &\ge \E\Big[\frac{1}{\gamma}\sum_{i=1}^N \big(\gamma \text{ } \sum_{t=1}^T p^i_t - \tilde{R}_{\W}(T) \big)\Big]
  \tag{Weak learning (\ref{online_agnostic_wl}, Lemma \ref{beygelzimer_lemma})} \\
    &\ge \sum_{i=1}^N \sum_{t=1}^T \E[p_t^i] - \frac{N}{\gamma} \tilde{R}_{\W}(T),
\end{align*}
Thus, we obtain the lower bound on the expected sum of losses $\sum_t \sum_i \ell^i_t(p^i_t)$ (see Line 6 in Algorithm~\ref{alg:online_agnostic} for the definition of the $\ell^i_t$'s), given by,
\begin{align*} 
\E[\sum_{t=1}^T \sum_{i=1}^N \ell^i_t(p^i_t)] &\ge - \frac{N}{\gamma} \tilde{R}_{\W}(T).
\end{align*}

For the upper bound, observe that the OCO regret guarantee implies that for any $t \in [T]$, and any $p^*_t \in [0,1]$, 
$$
    \E\Big[\frac{1}{N} \sum_{i=1}^N \ell^i_t(p^i_t)\Big] \le p^*_t \bigg(\bigg( \frac{1}{\gamma N}  \sum_{i=1}^N \E\big[\W_i(x_t)\big] \bigg)y_t -1 \bigg)   +  \frac{1}{N}R_{\mathcal{A}}(N),
$$
 Thus, by setting $p_t^*$ according to Lemma \ref{proj_lemma}, and summing over $t \in [T]$, we get,
$$
\E\Big[\frac{1}{N}\sum_{t=1}^T \sum_{i=1}^N \ell^i_t(p^i_t)\Big] \le \sum_{t=1}^T (\E[\hat{y}_t] y_t - 1) + \frac{T}{N}R_{\mathcal{A}}(N).
$$
By combining the lower and upper bounds for $\E\big[\frac{1}{NT}\sum_t\sum_i \ell^i_t(p^i_t)\big]$, we get,
$$
\frac{1}{T}\sum_{t=1}^T
\E[\hat{y}_t] y_t \ge  1 - \Bigg(\frac{R_\W(T)}{\gamma T} + \frac{R_\A(N)}{N}\Bigg).
$$

\vspace{2mm}

\subsection*{Proof of Theorem \ref{thm:statistical_realizable}}
\textit{Reduction to Proposition~\ref{game_prop}.}
Let $h^*$ be a concept consistent with the input sample (i.e.\ $h^*(x_i)=y_i$ for $i\leq m$) and let $\mathcal{H}'=\mathcal{H}\cup\{h^*\}$. It is convenient to define the decision sets are
defined by $\K_A = [0, 1]^m$, $\K_B = \Delta_{\mathcal{H}'}$, and $\K_B' = \frac{1}{\gamma}\Delta_{\mathcal{H}'}$,
and the payoff function $g(\cdot,\cdot)$ is again given by
\[g(p, q) = \sum_{i=1}^m p(i)(q(x_i)y_i - 1).\]  
The weak learner corresponds to an approximate optimization oracle $\W$ with no additive error. That is,
for any $p\in \K_A$ the output $q'=\W(p)$ satisfies $q'\in \K_B'$ and
\[
\E[g(p, q')] \ge 0.
\]

Next, one can show that the value of the game in this setting is $\lambda^*=0$: indeed, this follows simce $\lambda^* = \min_{p\in \K_A} g(p, q^*)=0$ and since the pure strategy supported on $h^*$, $q^*=q_{h^*} \in \K_B$ is dominant for player B.
Applying Proposition \ref{game_prop}, we have for any $p\in \K_A$, with $\bar{q}(x_i) = \frac{1}{\gamma T} \sum_{t=1} h_t(x_i) \in \K_B'$,
\begin{equation}\label{eq:4}
0\le \E[\sum_{i=1}^{m}p(i)(\bar{q}(x_i)y_i - 1)] + \frac{R_\A(T)}{T}.
\end{equation}

\vspace{2mm}
\textit{Projection.}
By the definition of $\bar{h}$, using Equation \ref{eq:4} and Lemma \ref{proj_lemma}, we have
$$
0 \le \E[\cor_S(\bar{h})] - 1 + \frac{R_\A(T)}{Tm}.
$$
As before, using OGD as the OCO algorithm $\A$ yields $\frac{R_\A(T)}{Tm} = O(\frac{1}{\gamma \sqrt{T}})$.

\end{document}